\documentclass{article}

\usepackage{microtype}
\usepackage{graphicx}
\usepackage{subfigure}
\usepackage{booktabs} %
\usepackage[ruled,lined,boxed, algo2e]{algorithm2e}
\usepackage{xcolor}

\usepackage{hyperref}

\usepackage[accepted]{icml2023}

\usepackage{amsmath}
\usepackage{amssymb}
\usepackage{mathtools}
\usepackage{amsthm}
\usepackage{multirow}

\usepackage[capitalize,noabbrev]{cleveref}

\usepackage{caption}

\theoremstyle{plain}
\newtheorem{theorem}{Theorem}[section]
\newtheorem{proposition}[theorem]{Proposition}

\theoremstyle{definition}
\newtheorem{definition}[theorem]{Definition}

\theoremstyle{remark}

\usepackage[textsize=tiny]{todonotes}

\usepackage{amsfonts}
\usepackage{amssymb}
\usepackage{amsthm}
\usepackage{hyperref}
\usepackage{mathtools}
\usepackage{microtype}
\usepackage{nicefrac}
\usepackage{url}
\usepackage{xcolor}
\usepackage{cleveref}  %

\newcommand\given{\,\vert\,}  %

\DeclarePairedDelimiterX{\divergence}[2]{(}{)}{#1\,\delimsize\|\,#2}

\newcommand{\E}[2]{\mathbb{E}_{#1}\left[#2\right]} %

\DeclarePairedDelimiter\abs{\lvert}{\rvert}  %
\newcommand\jurl[1]{\href{https://#1}{\nolinkurl{#1}}}

\DeclareMathOperator{\X}{\mathcal{X}}

\DeclareMathOperator{\BigO}{\mathcal{O}}

\newcommand{\Y}{\mathcal{Y}}

\newcommand{\changed}[1]{#1}
\newcommand{\changedtwo}[1]{#1}
\newcommand{\changedthree}[1]{#1}

\newcommand{\eclass}[1]{[#1]}

\begin{document}

\twocolumn[
\icmltitle{Random Edge Coding: One-Shot Bits-Back Coding of Large Labeled Graphs}
\begin{icmlauthorlist}
\icmlauthor{Daniel Severo}{uoft,vector}
\icmlauthor{James Townsend}{uva}
\icmlauthor{Ashish Khisti}{uoft}
\icmlauthor{Alireza Makhzani}{uoft,vector}
\end{icmlauthorlist}

\icmlaffiliation{uoft}{Department of Electrical and Computer Engineering, University of Toronto}
\icmlaffiliation{vector}{Vector Institute for Artificial Inteligence}
\icmlaffiliation{uva}{Amsterdam Machine Learning Lab (AMLab), University of Amsterdam}

\icmlcorrespondingauthor{Daniel Severo}{d.severo@mail.utoronto.ca}
\icmlcorrespondingauthor{Alireza Makhzani}{makhzani@vectorinstitute.ai}

\icmlkeywords{lossless compression, edge-exchangeable models, polyas urn, generative models, graph compression, bits-back coding, one-shot compression, entropy coding, ANS}

\vskip 0.3in
]

\printAffiliationsAndNotice{}  %

\begin{abstract}
\changedthree{
We present a one-shot method for compressing large labeled graphs called \emph{Random Edge Coding}.
When paired with a parameter-free model based on Pólya's Urn, the worst-case computational and memory complexities scale quasi-linearly and linearly with the number of observed edges, making it efficient on sparse graphs, and requires only integer arithmetic.
Key to our method is bits-back coding, which is used to sample edges and vertices without replacement from the edge-list in a way that preserves the structure of the graph.
Optimality is proven under a class of random graph models that are invariant to permutations of the edges and of vertices within an edge.
Experiments indicate Random Edge Coding can achieve competitive compression performance on real-world network datasets and scales to graphs with millions of nodes and edges.}
\end{abstract}

\section{Introduction}

\emph{Network data}, such as social networks and web graphs \cite{newman2018networks}, can be represented as large graphs, or multigraphs, with millions of nodes and edges corresponding to members of a population and their interactions.
To preserve the underlying community structure, as well as other relational properties, algorithms for compressing and storing network graphs must be \emph{lossless}.

\changed{
Entropy coding is a popular lossless compression technique that makes use of a probabilistic model over the observed data.
The main advantage of this approach over ad hoc methods is that it enables the use of existing domain knowledge on statistical modelling, which in the case of graphs is a well studied field (see Section \ref{sec:background-modeling}).
Bridging network modeling and data compression has yielded significant gains in compression performance with little additional effort by the research community \cite{steinruecken2014b}.
For example, BB-ANS \cite{townsend2019} allowed powerful latent variable models to be easily repurposed for lossless data compression including variational autoencoders \cite{townsend2019}, diffusion models \cite{kingma2021variational}, and integer discrete flows \cite{hoogeboom2019integer, berg2020idf++}, leading to state-of-the-art compression performance on image, speech \cite{havtorn2022benchmarking} and smart meter time-series \cite{jeong2022lossless} data.
In many domains, current state-of-the-art compression algorithms such as CMIX \cite{knoll2012machine} and DeepZip \cite{7fcb664b03ac4d6497048954d756b91f} use entropy coding together with a powerful neural network model.}

\begin{figure}[t]
    \centering
    \includegraphics[width=0.9\columnwidth]{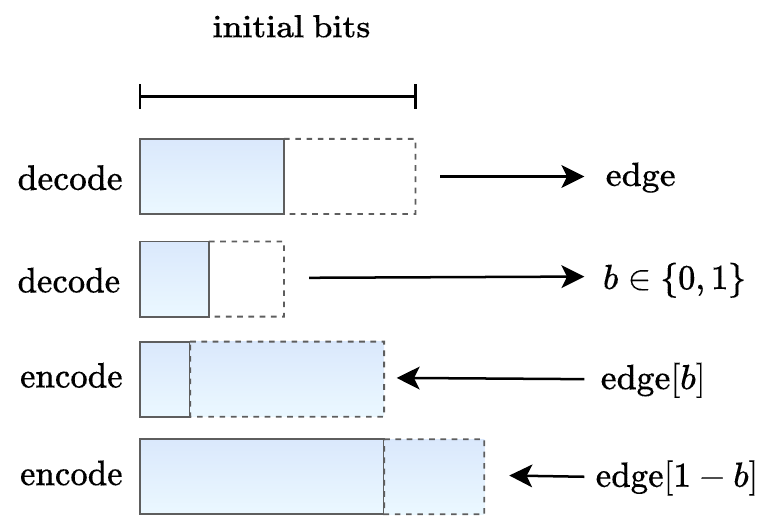}
    \caption{A diagram of our method, Random Edge Coding, compressing a single edge. 
    The ANS stack is represented in blue.
    First, a random edge is sampled without replacement from the edge-list using an ANS decode.
    Then, a binary indicator $b$ is decoded that defines the order in which vertices will be encoded: first $\text{edge}[b]$ then $\text{edge}[1 - b]$.}
    \label{fig:bb-exg}
\vskip -0.2in
\end{figure}

Compressing a single network is a \emph{one-shot} lossless compression problem.
This precludes the use of methods that require amortizing over repeated independent observations to reach the fundamental lower bound on lossless compression: the entropy \cite{cover1999elements}.
To minimize the average compression rate in the one-shot setting, the optimal number of bits to allocate for a single instance is equal to the negative log-likelihood (also known as the information content) under the estimated model.

The entropy coding of a network can, in general, be done in quadratic time with respect to the number of vertices by storing binary variables indicating the absence or presence of each possible edge in the graph.
However, most real-world networks are \emph{sparse} in the sense that the number of edges is significantly lower than the maximum number of possible edges.
Therefore, algorithms with sub-quadratic complexity in the number of observed edges are more attractive for compressing network data than those that scale quadratically with the number of nodes.

\changedthree{In this work we present a one-shot entropy coder for large labeled graphs called \emph{Random Edge Coding} (REC).}
Our method is optimal under a broad class of distributions referred to as \emph{edge-permutation invariant} (\Cref{def:fex}) and can achieve competitive performance on real-world networks as we show in \Cref{sec:experiments}.
When paired with Pólya's Urn \cite{mahmoud2008polya}, a parameter-free model described in \Cref{sec:polya}, our method requires only integer arithmetic and the worst-case computational and memory complexities scale quasi-linearly and linearly with the number of observed edges in the graph.
\changedthree{REC is applicable to both simple and non-simple labeled graphs, with directed or undirected edges, as well as hyper-graphs}

REC uses \emph{bits-back coding} \cite{frey1996free, townsend2019} to sample edges and vertices without replacement from the graph's edge-list, similarly to \cite{severo2021compressing, severo2021your}.
Sampling is done by decoding from a shared random seed, which also stores the final message (i.e. the bits of the graph).

\changed{
Recent methods for lossless compression of graphs are reviewed in \Cref{sec:related-work}.
We define the notation used throughout the paper in \Cref{sec:notation}.
In \Cref{sec:background} we discuss parameter-efficient models that yield good likelihood values on large sparse networks for which REC is optimal, as well as give an introduction to entropy coding.
REC is introduced in \Cref{sec:method} and, together with Pólya's Urn, is shown to achieve compression results competitive with the state-of-the-art on real-world network datasets in \Cref{sec:experiments}.}

\section{Related Work}\label{sec:related-work}
To the best of our knowledge there is no previous entropy coding method that can scale to large graphs and is optimal for the broad class of edge-permutation invariant (EPI) graphs of \Cref{def:fex}.

A number of previous works have presented adhoc methods for lossless compression of large graphs including Pool Compression \cite{yousuf2022pool}, SlashBurn \cite{lim2014slashburn}, List Merging \cite{grabowski2014tight-LM}, BackLinks \cite{chierichetti2009compressing-BL}, and Zuckerli \cite{versari2020zuckerli}.
In sum, these methods attempt to exploit local statistics of the graph edge-list by defining an ordering of the vertex sequence that is amenable to compression.
See \citet{yousuf2022pool} for an overview of the methods.
Re-ordering techniques would yield no effect for EPI models as all permutations of the edge sequence have the same likelihood, as discussed in \Cref{sec:fee}.
In \Cref{sec:experiments}, \Cref{table:results} we compare the performance of these methods with that of entropy coding under \changed{Pólya's Urn model} using our method and show that it performs competitively and can even outperform previous methods on sparser datasets.

\citet{chen2021order} develop a deep latent variable model where the edges are observations and the ordering of the vertices are latent.
This could be made into a compression algorithm for a dataset of graphs by combining it with the entropy coder developed in \cite{townsend2019}.
It is unclear if this could be used for one-shot compression, where only a single graph is available.
Furthermore, the authors present results on small graphs with a few hundred nodes and edges, while our setting is that of millions.

Another machine learning method for lossless graph compression is Partition and Code \cite{bouritsas2021partition}.
The method decomposes the graph into a collection of subgraphs and performs gradient descent to learn a dictionary code over subgraphs.
While achieving good compression performance on small graph datasets, it is unclear if these methods can scale to networks with millions of nodes and edges.

Our work draws upon a large body of work on the statistical modeling of networks \cite{newman2018networks, bloem2017random, crane2018edge} which is reviewed in \Cref{sec:background-modeling}.

\section{Notation \changedtwo{and Setup}}\label{sec:notation}

We use $[n]$ to represent a set of integers $\{1, \dots, n\}$ and use superscript to represent sequences such as $x^n = x_1, \dots, x_n$.

\changedtwo{All graphs in this work are labeled, have a fixed number of nodes ($n$), a variable number of edges ($m$), and are in general non-simple (i.e., allow loops and repeated edges).}

Graph sequences are represented similarly to \cite{bloem2017random}.
A graph sequence of $n$ nodes, where an edge is added at each step, can be represented as a sequence of vertex elements $v_i$ taking on values in the alphabet $[n]$ with the $i$-th edge defined as $e_i = (v_{2i-1}, v_{2i})$.
We refer to the edge and vertex sequences interchangeably.
\changed{The $i$-th graph is taken to be $G_i = \{\{v_1, v_2\}, \dots, \{v_{2i-1}, v_{2i}\}\}$ and the $k$-th vertex of an edge is indicated by $e[k]$.}

The likelihood, and hence the information content, of $v^{2m}$ and $e^m$ are the same, but they differ from that of $G_m$.
A sequence of vertices or edges carries the information regarding the order in which the vertices and edges appeared in the graph sequence, which may be significantly larger than that of the graph.

\changedthree{
The ascending factorial function $a: \mathbb{R}\times\mathbb{N} \mapsto \mathbb{R}$ is defined by $a(x, k) = x(x+1)(x+2)\dots(x+k-1)$, for $k > 0$, with $a(x, 0) = 1$ for all $x$, and is abbreviated as $x^{\uparrow k}$.
}

All distributions are discrete, logarithms are base 2, and we refer to the probability mass function and cumulative distribution function as the PMF and CDF.

\section{Background}\label{sec:background}
\subsection{Modeling Real-World Networks}\label{sec:background-modeling}
Most real-world networks are \emph{sparse} in the sense that the number of edges $m$ is significantly smaller than the maximum number of possible edges $\binom{n}{2}$ in a simple graph \cite{newman2018networks}.
A graph that is not sparse is known as a \emph{dense} graph.
Compression algorithms with complexity that scales with the number of edges instead of nodes are therefore more efficient for real-world network graphs.

Some networks exhibit \emph{small-world} characteristics where most nodes are not connected by an edge but the degree of separation of any 2 nodes is small \cite{newman2018networks}.
The degree distributions are heavy-tailed due to the presence of hubs, i.e., vertices with a high degree of connectivity.
Random graph models that assign high likelihood to graphs with small-world characteristics are thus preferred to model and compress these network types.

\changedthree{
\subsection{Random Graph Models}\label{sec:background-inf-models}

Modeling random graphs is a well studied field dating back to the early work of Erdős, Rényi, and Gilbert \cite{erdHos1960evolution} where either graphs with the same number of edges are equally likely or an edge is present in the graph with a fixed probability $p$.
The field has since evolved to include the stochastic block model \cite{holland1983stochastic}, where the edge probability is allowed to depend on its endpoints, as well as its mixed-membership variant \cite{airoldi2008mixed}.
More recently, \cite{caron2017sparse, cai2016edge, crane2018edge} have found some success in modeling real-world network graphs.
These models have been used in a number of applications including clustering \cite{sewell2020model}, anomaly detection \cite{luo2021anomalous}, link-prediction \cite{JMLR:v17:16-032}, community detection \cite{zhang2022node}, and have been extended to model hierarchical networks \cite{dempsey2021hierarchical}.

The model used in this work assigns likelihood to a vertex- or edge-sequence autoregressively.
Neural network models have also been used for autoregressive graph modeling such as \cite{you2018graphrnn, bacciu2020edge, goyal2020graphgen}. See \cite{zhu2022survey} for a survey.
The likelihood assigned by these models usually depend on the order in which vertices or edges were added to the graph, in contrast to the Pólya's Urn based-model used in this work which is order-invariant.

}

\subsection{Lossless Compression with Entropy Coding}
\label{sec:background-entropy-coding}
Given a discrete distribution $P$ over alphabet $\X$, the objective of lossless compression is to find a code $C\colon\X \mapsto \{0,1\}^\star$ that minimizes the average code-length $\E{x \sim P}{\ell(C(x))}$.
Shannon's source coding theorem \cite{cover1999elements} guarantees that the average code-length is lower bounded by the \emph{entropy} $H(P) = \E{x \sim P}{-\log P(x)}$.
A code with average code-length within 1 bit of the entropy is called \emph{optimal} and assigns a code-word to $x \in \X$ with length close to its \emph{information content} $\ell(C(x)) \approx - \log P(x)$.

\emph{Asymmetric Numeral Systems} (ANS) is an algorithm that can implement an optimal code for $(\X, P)$ \cite{duda2009}.
ANS maps an instance $x \in \X$ to an integer state $s$ via an encoding function.
The instance can be losslessly recovered from the state via a decoding function which inverts the encoding procedure.
Both encoding and decoding functions require access to the PMF and CDF as well as the current state.
The last symbol encoded during compression is the first to be decoded during decompression implying ANS is \emph{stack-like} (in contrast to most other entropy coders which are \emph{queue-like}).
For a more detailed discussion on ANS see \cite{townsend2020, townsend2021lossless}.

A key point to our method is that ANS can be used as a sampler by initializing the state to a random integer and performing successive decode operations with the distribution one wishes to sample from.
This operation is invertible in the sense that the random seed can be fully recovered by performing an ANS encode with the sampled symbols.
Decoding with distribution $P$ reduces the size of the ANS state by approximately $\log 1/P(x)$ bits, where $x$ is the sampled symbol, while encoding increases the state by the same amount.
This is illustrated in \Cref{fig:bb-exg}.

\subsection{Bits-Back coding with ANS}\label{sec:bits-back-coding}
Bits-Back coding with ANS (BB-ANS) is a general entropy coding method for latent variable models \cite{townsend2019}.
Given a model with observations $x$ and latents $z$ defined by an approximate posterior $Q(z \given x)$, conditional likelihood $P(x \given z)$ and prior $P(z)$, BB-ANS can perform tractable lossless compression of a sequence of observations.

In sum, BB-ANS projects the data $x$ into the extended state-space $(x, z)$ by sampling from $Q(z \given x)$ while using the ANS state as a random seed and then encoding $(x, z)$ with the joint distribution $P(x \given z)P(z)$ \cite{ruan2021}.
The resulting ANS state is then used as a random seed to compress the next observation using the same procedure.
Sampling with $Q(z \given x)$ results in a reduction of the number of bits in the ANS state.
The average number of bits achieved by BB-ANS is a well known quantity called the (negative) \emph{Evidence Lower Bound (ELBO)}  which is an upper bound on the entropy of the data \cite{jordan1999introduction}.
This method has been extended to hierarchical latent variable models \cite{kingma2019bit} and state space models \cite{townsend2021a}.

Our method can be understood under this framework as constructing a latent variable model with an exact approximate posterior, i.e., $Q(z \given x) = P(z \given x)$.
The distribution of the latent variable in our method is a Markov chain of sub-graphs of the graph being compressed.

The first step to encode a symbol is to decode a latent variable with the approximate posterior.
As mentioned, this yields a savings of approximately $\log 1/Q(z \given x)$ bits per observation.
However, when encoding the first observation, the ANS stack is initially empty and therefore $\log 1/Q(z \given x)$ bits must be encoded into the ANS state to allow for sampling.
This initial overhead is known as the \emph{initial bits problem} \cite{townsend2019} and in general is amortized by compressing multiple observations.

\subsection{Entropy Coding for Large Alphabets}
\label{sec:background-ec-large-alphabets}
A large alphabet can render entropy coding intractable even if the PMF is tractable.
For example, the alphabet size $\abs{\X}$ of a graph grows exponentially with the number of nodes.
Storing the PMF and CDF values in a hashtable provides fast look-ups but will in general require $\BigO(\abs{\X})$ memory which is impractical.

A common strategy to trade-off computational and memory complexity is to avoid compressing $(\X, P)$ directly and instead construct a one-to-many mapping, possibly a bijection, between $\X$ and an alphabet of proxy sequences $y^m$.
For example, we can decompose a graph into a sequence of $\binom{n}{2}$ edge incidence variables each with alphabet sizes $\abs{\Y_i} = 2$.
Compression is then performed autoregressively with $y_i$ conditioned on the history $y^{i-1}$ and only the CDF and PMF of the current variable are stored in memory.
The worst-case computational complexity is $\BigO\left(\sum_{i=1}^m c_i\right)$, where $c_i$ represents the complexity of computing the PMF and CDF of $y_i$, while memory is $\BigO(1)$.
If $c_i$'s can be kept small then the trade-off can be made useful for practical applications.

Our method can be seen as an instance of this technique where the graph is mapped to its edge-sequence with a random order.

\section{Method}\label{sec:method}
    \begin{figure}[t]
        \centering
        \includegraphics[width=1.0\columnwidth]{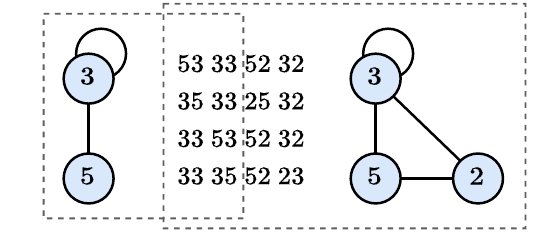}
        \caption{Example of the correspondence between vertex sequences and graph.
        Middle: 4 vertex sequences all corresponding to the same graphs. The vertices are labeled by single-digit integers and have been grouped to highlight the edge they belong to.
        Right: The graph corresponding to all 4 sequences.
        Left: The graph corresponding to the sub-sequences made up of the first 2 edges. In this particular case, all elements of the equivalence class of the graph are shown (i.e., all 4 sub sequences 53 33, 35 33, 33 53, and 33 35).
        For edge-permutation invariant models (\Cref{def:fex}), equivalent sequences are assigned the same likelihood.}
        \label{fig:seq-to-graph}
        \vskip -0.2in
    \end{figure}
    In this section we present our entropy coder \emph{Random Edge Coding} (REC) which is, to the best of our knowledge, the first to perform optimal one-shot compression of a broad class of large graphs.
    
   \changedtwo{Our method is optimal for any model over graphs satisfying edge-permutation invariance (\Cref{def:fex}).
    \Cref{sec:polya} discusses an example of such a model, known as Pólya's Urn \cite{mahmoud2008polya}, which is parameter-free.
    This model is used in \Cref{sec:experiments} to achieve competitive performance on network data.}
    
    See \Cref{sec:notation} for a review of the notation used.

    \subsection{\changedtwo{Vertex- and Edge-Permutation Invariance}}\label{sec:fee}
    REC is an optimal entropy coder for PMFs over vertex sequences that are invariant to permutations of the edges and of vertices within an edge.
    This is characterized formally by the following definition.

    \begin{definition}[Edge-Permutation Invariance (EPI)]\label{def:fex}
    Let $v^{2m}$ be a vertex sequence with edges defined as $e_i = (v_{2i-1}, v_{2i})$ and $\sigma$ an arbitrary permutation function over $m$ elements.
    Given a collection $(\pi_k)_{k=1}^m$ of permutation functions over 2 elements, each over integers $(2k-1, 2k)$, we say that a model is \emph{edge-permutation invariant} if the following holds
    \begin{align}
        \Pr(e_1, \dots, e_m) = \Pr(\tilde{e}_{\sigma(1)}, \dots, \tilde{e}_{\sigma(m)}),
    \end{align}
    where 
    \begin{align}
        \tilde{e}_k = (v_{\pi_k(2k-1)}, v_{\pi_k(2k)}).
    \end{align}
    \end{definition}    

    \changedtwo{
    A stronger property that implies EPI is that of \emph{vertex-permutation invariance}, which coincides with the common definition of finite exchangeability of sequences. 
    
    \begin{definition}[Vertex-Permutation Invariance (VPI)]\label{def:fvx}
    Let $v^{2m}$ be a vertex sequence and $\pi$ an arbitrary permutation function over $2m$ elements.
    We say that a model is \emph{vertex-permutation invariant} if the following holds
    \begin{align}
        \Pr(v_1, \dots, v_{2m}) = \Pr(v_{\pi(1)}, \dots, v_{\pi(2m)}),
    \end{align}
    for all permutations $\pi$.
    \end{definition}

    The model we present next is VPI and therefore also EPI.
    }

    \subsection{Pólya's Urn Model}\label{sec:polya}

    Pólya's Urn (PU) model \cite{mahmoud2008polya} defines a joint probability distribution over a vertex sequence $v^k$ that is VPI.
    The generative process of PU is as follows.
    An urn is initialized with $n$ vertices labeled from $1$ to $n$.
    At step $i$, a vertex is sampled from the urn, assigned to $v_i$, and then returned to the urn together with an extra copy of the same vertex.    
    
    The joint PMF is defined via a sequence of conditional distributions
    \begin{align}\label{eq:fhm-conditional}
        \Pr(v_{i+1} \given v^i) \propto d_{v^i}(v_{i+1}) + 1,
    \end{align}
    where $d_{v^i}(v) = \sum_{j=1}^i 1\{v = v_j\}$ is the degree of vertex $v$ in $v^i$.
    The joint distribution can be expressed as
    \begin{align}\label{eq:polya-joint}
        \Pr(v^k) = \frac{1}{n^{\uparrow k}}\prod_{v \in [n]} d_{v^k}(v)!,
    \end{align}
    where equation \eqref{eq:polya-joint} depends only on the degrees of the vertices in the final sequence $v^k$, guaranteeing VPI.
    We provide a complete derivation in \Cref{appendix:proof-polya}.

    Although the generative process is fairly simple, the resulting joint distribution can achieve competitive likelihood values for real-world networks with small-world characteristics as indicated in \Cref{table:results}.
    To understand why, note that equation \eqref{eq:polya-joint} assigns higher likelihood to graphs with non-uniform degree distributions.
    The product of factorial terms is the denominator of the multinomial coefficient, and hence is largest when some vertex dominates all others in degree.

    PU is parameter-free and therefore requires $0$ bits to be stored and transmitted.
    The PMF and CDF of PU, needed for entropy coding with ANS, can be computed with integer arithmetic. 
    It therefore does not suffer from floating-point rounding errors, which is common in neural compression algorithms \cite{yangneural2023, balle2019integer}.
    \subsection{\changedthree{Random Edge Coding} (REC)}
    \changedtwo{
    In this section we describe our method for the case of simple graphs (i.e., no loops or repeated edges).
    The extension to non-simple graphs is explained in \Cref{sec:method-non-simple}, while hyper and direct graphs are handled in \Cref{sec:hypergraphs}.
    }
    
    Compressing the graph with entropy coding requires computing the PMF and CDF of the graph from the PMF of the vertex sequence.
    Although the PMF of the graph \changed{for PU has a} closed form expression, the alphabet size grows exponentially with the number of nodes which makes direct entropy coding infeasible (see \Cref{sec:background-ec-large-alphabets}).

    We employ the technique discussed in \Cref{sec:background-ec-large-alphabets} to trade-off computational and memory complexity by mapping the graph to a sequence of equivalence classes containing vertex sequences.

    The vertex equivalence class $\eclass{v^{2m}}$ of a graph $G$ with $m$ edges is the set of all vertex sequences that map to $G$ (see \Cref{fig:seq-to-graph} for an example).
    Models satisfying \Cref{def:fex} assign equal likelihood to sequences in the same equivalence class.
    Therefore, the negative loglikelihood of $G, \eclass{v^{2m}}$, and $v^{2m}$ are related by
    \begin{align}\label{eq:info-content}
        \log1/\Pr(G) 
        &= \log1/\Pr(\eclass{v^{2m}})  \\
        &= \log1/\Pr(w^{2m}) - \log\abs{\eclass{w^{2m}}},
    \end{align}
    for any equivalent $w^{2m}, v^{2m}$ that map to graph $G$.
    The size of the equivalence class can be computed by counting the number of edge-permutations and vertex-permutations within an edge, which add up to
    \begin{align}\label{eq:savings}
        \log\abs{\eclass{w^{2m}}} = m + \log m!.
    \end{align}
    The relationship between the likelihoods implies that we can reach the information content of the graph by compressing one of its vertex sequences if we can somehow get a number of bits back equal to $\log\abs{\eclass{w^{2m}}}$.
    This leads to the naive \Cref{alg:naive}, which we describe below, that suffers from the initial bits issue (see \Cref{sec:bits-back-coding}).    
    
    \begin{algorithm}[h]
       \caption{Naive Random Edge Encoder}
       \label{alg:naive}
    \begin{algorithmic}
       \STATE {\bfseries Input:} Vertex sequence $v^{2m}$ and ANS state $s$.
       \STATE 1) Edge-sort the vertex sequence $v^{2m}$
       \STATE 2) Decode a permutation $\sigma$ uniformly w/ prob. $1/\abs{\eclass{v^{2m}}}$
       \STATE 3) Apply the permutation to the vertex sequence
       \STATE 4) Encode the permuted vertex sequence
       \STATE 5) Encode $m$ using $\log m$ bits
    \end{algorithmic}
    \end{algorithm}
    \changed{
    At step 1) we sort the vertex sequence without destroying the edge information by first sorting the vertices within an edge and then sorting the edges lexicographically.
    For example, edge-sorting sequence $(34\ 12\ 32)$ yields $(12\ 23\ 34)$.
    }
    In step 2) an index is decoded that corresponds to a permutation function agreed upon by the encoder and decoder, which is applied to the sequence in step 3).
    Note these permutations do not destroy the edge information by design.
    Finally, in step 4), the permuted sequence is encoded using $\Pr(v_{i+1} \given v^i)$ from \Cref{sec:method} followed by the number of edges.
    Decoding a permutation reduces the number of bits in the ANS state by exactly $\log\abs{\eclass{v^{2m}}}$, while encoding the vertices increases it by $\log1/\Pr(v^{2m})$.
    From \eqref{eq:info-content}, the net change is exactly the information content of the graph: $\log1/\Pr(G)$.

    The decoder acts in reverse order and perfectly inverts the encoding procedure, restoring the ANS state to its initial value.
    First, $m$ is decoded.
    Then the sequence is decoded and the permutation is inferred by comparing it to its sorted version.
    Finally, the permutation is encoded to restore the ANS state.
    
    Unfortunately, this method suffers from the initial bits problem \cite{townsend2019}, as the decode step happens before encoding, implying there needs to be existing information in the ANS stack for the bit savings to occur.
    It is possible to circumvent this issue by incrementally sampling a permutation, similarly to \cite{severo2021compressing, severo2021your}.
    This yields \Cref{alg:bb-exg}, which we describe below, and is illustrated in \Cref{fig:bb-exg}.

    \begin{algorithm}[h]
       \caption{Random Edge Encoder}
       \label{alg:bb-exg}
    \begin{algorithmic}
       \STATE {\bfseries Input:} Vertex sequence $v^{2m}$ and ANS state $s$.
       \STATE 1) Edge-sort the vertex sequence $v^{2m}$\\
       \REPEAT
       \STATE 2) Decode an edge $e_k$ uniformly from the sequence
       \STATE 3) Remove $e_k$ from the vertex sequence
       \STATE 4) Decode a binary vertex-index $b$ uniformly in $\{0, 1\}$
       \STATE 5) Encode $e_k[b]$
       \STATE 6) Encode $e_k[1-b]$
       \UNTIL{The vertex sequence is empty}
       \STATE 7) Encode $m$ using $\log m$ bits.
    \end{algorithmic}
    \end{algorithm}

    REC progressively encodes the sequence by removing edges in a random order until the sequence is depleted.
    As before, we edge-sort the vertex sequence in step 1) without destroying the edge information.
    Then, in steps 2) and 3), an edge is sampled without replacement from the sequence by decoding an integer $k$ between $1$ and the size of the remaining sequence.
    Since the graph is undirected, we must destroy the information containing the order of the vertices in the edge.
    To do so, in step 4), we decode a binary index $b$ and then encode vertices $e_k[b], e_k[1-b]$ in steps 5) and 6) using $\Pr(v_{i+1} \given v^i)$.
    Finally, $m$ is encoded.

    The initial bits overhead is amortized as the number of edges grows.
    This makes REC an optimal entropy coder for large EPI graphs, as characterized by the proposition below.

    \begin{proposition}
        Random Edge Coding is an optimal entropy coder for any edge-permutation invariant graph model (\Cref{def:fex}) as $m \rightarrow \infty$.
    \end{proposition}
    \begin{proof}
        An optimal entropy coder compresses an object to within one bit of its information content (i.e., negative log-likelihood).
        Encoding steps add $\log1/\Pr(v_{i+1} \given v^i)$ bits to the ANS state resulting in a total increase of $\log1/\Pr(v^{2m})$ bits.
        Each decoding operation removes bits from the ANS state and together save $\sum_{i=1}^m (1 + \log i) = \log1/\Pr(\abs{\eclass{v^{2m}}})$.
        From \eqref{eq:info-content}, the net change is exactly the information content of the graph: $\log1/\Pr(G)$.
        The initial and $\log m$ bits (needed to encode $m$) are amortized as $m \rightarrow \infty$.
        Therefore, the number of bits in the ANS state approaches \eqref{eq:info-content}, which concludes the proof.
    \end{proof}

    In \Cref{sec:experiments} and \Cref{table:optimality} we show empirical evidence for the optimality of REC by compressing networks with millions of nodes and edges down to their information content under \changed{Pólya's Urn model}.
    \changedtwo{
    \subsection{Extension to Non-Simple Graphs}\label{sec:method-non-simple}
    If the graph is non-simple then the size of the equivalence class $\eclass{w^{2m}}$ will be smaller.
    The savings in \eqref{eq:savings} must be recalculated by counting the number of valid permutations of edges, and vertices within an edge, that can be performed on the sequence.
    Furthermore, \Cref{alg:bb-exg} must be modified to yield the correct savings.
    Handling repeated loops and repeated edges requires different modifications which we discuss below.
    
    Each non-loop edge doubles the size of the equivalence class, while loops do not as the vertices are indistinguishable and thus permuting them will not yield a different sequence.
    This can be handled by skipping step 4) and setting $b=0$ if $e_k[0] = e_k[1]$.
    
    In general, the number of possible edge-permutations in a non-simple graph $G_m$ with $m$ undirected edges is equal to the multinomial coefficient}
    \begin{align}
        \binom{m}{c_1, c_2, \dots} = \frac{m!}{\prod_{e \in G_m} c_e!}\leq m!,
    \end{align}
    where $c_e$ is the number of copies of edge $e$ and $e \in G_m$ iterates over the \emph{unique} edges in $G_m$.
    Equality is reached when there are no repeated edges ($c_e=1$ for all edges).
    
    To achieve this saving, \Cref{alg:bb-exg} must be modified to sample edges uniformly from the graph \emph{without replacement}.
    In other words, step 2) is generalized to sample $e_k=e$ with probability $c^k_e/k$, where $c^k_e$ are the number of remaining copies of edge $e$ at step $k$.
    The count $c^k_{e}$ is non-increasing for all $e$ due to step 3) which, together with step 2), implements sampling without replacement.
    Furthermore, since all edges will eventually be decoded, the product of counts $\prod_{k=1}^m c^k_{e_k}$ contains all terms appearing in the factorial $c_e!$ for all edges $e$.
    The saving at each step is $-\log c^k_{e_k}/k$ and together will equal the log of the multinomial coefficient
    \begin{align}
    -\log\prod_{k=1}^{m} \frac{c^k_{e_k}}{k} = \log\frac{m!}{\prod_{e \in G} c_e!}.
    \end{align}

    These modifications together guarantee that REC is optimal for non-simple graphs.

    \subsection{Extension to Hypergraphs and Directed Graphs}\label{sec:hypergraphs}
    REC can be trivially extended to hypergraphs, where edges can have more than 2 nodes, directed graphs, and directed hypergraphs.

    For directed graphs, we need only ignore step 4) in \Cref{alg:bb-exg} and fix $b=0$.
    This guarantees the order information between vertices in an edge is preserved.

    For hypergraphs, steps 4-6) are generalized to the same sampling-without-replacement mechanism of step 3).
    After decoding an edge, the algorithm decodes vertices without replacement until the edge is depleted and encodes them in the order they appear.    
    It then proceeds to decode another edge and repeats these 2 steps until the sequence is depleted.

    \subsection{Complexity Analysis}
    For a graph with $m$ edges, the worst-case computational complexity of encoding and decoding with REC under \changed{Pólya's Urn model} is quasi-linear in the number of edges, $\BigO(m\log m)$, while the memory is linear: $\BigO(m)$.

    We discuss only encoding with REC as decoding is analogous.
    The first step during encoding is to sort the edge-sequence which has worst-case complexity $\BigO(m\log m)$.
    Then, the edge-list is traversed and the frequency count of all vertices are stored in a binary search tree (BST) with at most $2m$ elements ($\BigO(m)$ memory).
    The BST allows for worst-case look-ups, insertions, and deletions in $\BigO(\log m)$, which are all necessary operations to construct $\Pr(v_{i+1} \given v^i)$, as well as the CDF, used during entropy coding.
    Traversing the edge-list, together with the updates to the BST, require $\BigO(m \log m)$ computational complexity in the worst-case.

    \subsection{\changedthree{Random Edge Coding with Non-EPI Models}}
    While REC is only optimal for EPI models, it can still be paired with any probability model over vertex sequences that have well defined conditional distributions such as \eqref{eq:fhm-conditional}.

    For models that are not EPI the order of the vertices will affect the likelihood assigned to the graph.
    REC in its current form will discount at most $m + \log m!$ bits (with equality when the graph is simple) and all vertex sequences will have equal probability of appearing.
    The selected sequence $v^k$ will be determined by the initial bits present in the ANS stack (see \Cref{sec:bits-back-coding}).
    The number of bits needed to store the graph (i.e. information content) will therefore be,
    \begin{align}
        \log 1/\Pr(v^{2m}) - \left(\tilde{m} + \log \binom{m}{c_1, c_2, \dots}\right),
    \end{align}
    where $\tilde{m}$ is the number of non-loop edges and $v^{2m}$ the random sequence selected via the sampling-without-replacement mechanism of REC.

\section{Experiments}
\label{sec:experiments}
\begin{figure*}[t]
\centering
\begin{minipage}[t]{0.8\textwidth}
\captionof{table}{
Optimality of Random Edge Coding (REC) with \changed{Polya's Urn (PU) model}.
\textbf{REC achieves the optimal theoretical value for one-shot lossless compression under \changed{PU} for various datasets.}
Compressing the vertex sequence under the same model without REC would require a significantly higher number of bits-per-edge as indicated by the sequence's negative log-likelihood (NLL).
All units are in bits-per-edge.}
\vskip -0.2in
\label{table:optimality}
\vskip 0.15in
\begin{center}
\begin{small}
\begin{sc}
\begin{tabular}{rcccccc}
\toprule
Network     & Seq. NLL      & REC (Ours) & Graph NLL (Optimal)         & Gap (\%) \\
\midrule 
YouTube     & 37.91         & 15.19         & 15.19                  & 0.0   \\
FourSquare  & 31.14         & 9.96          & 9.96                   & 0.0   \\
Digg        & 32.67         & 10.62         & 10.62                  & 0.0   \\
Gowalla     & 32.11         & 12.19         & 11.69                  & 4.3   \\
Skitter     & 37.22         & 14.26         & 14.26                  & 0.0   \\
DBLP        & 35.48         & 15.92         & 15.92                  & 0.0   \\
\bottomrule
\end{tabular}
\end{sc}
\end{small}
\end{center}
\vskip -0.1in
\end{minipage}
\end{figure*}
\begin{figure*}[t]
\centering
\begin{minipage}[t]{0.9\textwidth}
\captionof{table}{Lossless compression of real-world networks.
\changed{Pólya's Urn (PU)} model, together with Random Edge Coding (REC), achieves competitive performance on some datasets and can even outperforms the current state-of-the-art on sparser social networks (black columns to the left).
\textbf{The last 2 columns (in gray) highlight the situations where \changed{PU} is expected to under-perform.}
While the compression with REC is optimal, the final results depend on the likelihood assigned to the graph under \changed{PU}, which is why ad hoc methods can achieve a better performance.
The best results are highlighted in bold.
Results for methods beyond ours are the ones reported by \cite{yousuf2022pool} in Table 4.
Units for the bottom section are in bits-per-edge and \textbf{lower numbers indicate better performance.}
}

\vskip -0.2in
\label{table:results}
\vskip 0.15in
\begin{center}
\begin{small}
\begin{sc}
\begin{tabular}{rcccccc}
\toprule
                          & \multicolumn{4}{c}{Social Networks}                                       & \multicolumn{2}{c}{\color{gray}{Others}} \\
                          & YouTube        & FourSq.       & Digg           & Gowalla        & \color{gray}{Skitter}       & \color{gray}{DBLP         }       \\
\midrule                                    
\# Nodes                  & 3,223,585      & 639,014       & 770,799        & 196,591        & \color{gray}{1,696,415 }    & \color{gray}{317,080      }       \\
\# Edges                  & 9,375,374      & 3,214,986     & 5,907,132      & 950,327        & \color{gray}{11,095,298}    & \color{gray}{1,049,866    }       \\
$10^6\times$Density       & 1.8            & 15.8          & 19.8           & 50.2           & \color{gray}{7.7       }    & \color{gray}{20.9         }       \\
\midrule                            
(Ours) PU w/ REC  & \textbf{15.19} & 9.96 & 10.62 & 12.19 & \color{gray}{14.26}         & \color{gray}{15.92}      \\
Pool Comp.                & 15.38 & \textbf{9.23} & 11.59          & \textbf{11.73} & \color{gray}{7.45 }         & \color{gray}{\textbf{8.78}}     \\
Slashburn                 & 17.03          & 10.67         & \textbf{9.82}  & 11.83          & \color{gray}{12.75}         & \color{gray}{12.62        }     \\
Backlinks                 & 17.98          & 11.69         & 12.56          & 15.56          & \color{gray}{11.49}         & \color{gray}{10.79        }     \\
List Merging              & 15.80          & 9.95          & 11.92          & 14.88          & \color{gray}{\textbf{8.87}} & \color{gray}{14.13        }      \\
\bottomrule
\end{tabular}
\end{sc}
\end{small}
\end{center}
\vskip -0.1in
\end{minipage}
\end{figure*}

In this section, we showcase the optimality of REC on large graphs representing real-world networks.

\changed{
We entropy code with REC using Pólya's Urn (PU) model and compare the performance to state-of-the-art compression algorithms tailored to network compression.}
We report the average number of bits required to represent an edge in the graph (i.e., bits-per-edge) as is common in the literature.

We used datasets containing simple network graphs with small-world statistics (see \Cref{sec:background-modeling}) such as YouTube, FourSquare, Gowalla, and Digg \cite{rossi2015network} which are expected to have high likelihood under \changed{PU}.
As negative examples, we compress Skitter and DBLP networks \cite{snapnets}, where we expect the results to be significantly worse than the state of the art, as these networks lack small-world statistics.
The smallest network (Gowalla) has roughly $200$ thousand nodes and $1$ million edges, while the largest (YouTube) has more than $3$ million nodes and almost $10$ million edges.

We use the ANS implementation available in Craystack \cite{townsend2020hilloc, craystack}.
The cost of sending the number of edges $m$ is negligible but is accounted for in the calculation of the bits-per-edge by adding $32$ bits.

To compress a graph, the edges are loaded into memory as a list where each element is an edge represented by a tuple containing two vertex elements (integers).
At each step, an edge is sampled without replacement using an ANS decode operation as described in \Cref{alg:bb-exg}.
Encoding is performed in a depth-first fashion, where an edge is encoded to completion before moving on to another.
\changed{Then, a vertex is sampled without replacement from the edge and entropy-encoded using \eqref{eq:fhm-conditional}.}
The process repeats until the edge is depleted, and then starts again by sampling another edge without replacement.
The process terminates once the edge-list is empty, concluding the encoding of the graph.
Decoding is performed in reverse order and yields a vertex sequence that is equivalent (i.e., maps to the same graph as) the original graph.

\changed{Code implementing Random Edge Coding, Pólya's Urn model, and experiments are available at \url{https://github.com/dsevero/Random-Edge-Coding}.}

\subsection{Optimality of Random Edge Coding}
We showcase the optimality of REC by compressing real-world graphs to the information content under the \changed{Pólya's Urn (PU) model} (see \Cref{sec:background-entropy-coding}).

\changedtwo{\Cref{table:optimality} shows the negative loglikelihood (NLL) of the vertex sequence and graph under PU.}
As discussed in \Cref{sec:background-entropy-coding} the graph's NLL is the value an optimal entropy coder should achieve to minimize the average number of bits with respect to the model.
REC can compress the graph to its NLL (as indicated by the last column of \Cref{table:optimality}) in all datasets except Gowalla.
Compressing the graph as a sequence of vertices (i.e., without REC) would require a number of bits-per-edge equal to the sequence's NLL, which is significantly higher than the NLL of the graph as can be seen by the first column of \Cref{table:optimality}.

As these methods evolve to achieve better likelihood values the compression performance is expected to improve automatically due to the optimality of REC.

\subsection{Compressing Real-World Networks}
In \Cref{table:results} we compare the bits-per-edge achieved by \changed{PU} using REC with current state-of-the-art algorithms for network data.
\changed{PU} performs competitively on all social networks and can even outperform previous works on networks such as YouTube \cite{rossi2015network}.

The likelihood assigned by \changed{PU} for non-social networks is expected to be low, resulting in poor compression performance, as indicated by the last 2 gray columns of \Cref{table:results}.
The performance of \changed{PU} deteriorates as the edge density increases and is visible from \Cref{table:results}.

While the compression with REC is optimal for \changed{PU}, the final results depend on the likelihood assigned to the graph under \changed{PU}, which is why ad hoc methods can achieve better performance.
Nonetheless, the bits-per-edge of \changed{PU} with REC is close to that of current methods.

The graph NLL for the Gowalla network under \changed{PU} shown in \Cref{table:optimality} ($11.68$) is less than the best compression result ($11.73$) achieved by Pool Compression \cite{yousuf2022pool}.
Assuming the network will grow with similar statistics, \changed{PU} may eventually surpass Pool Compression as there would be more edges to amortize the overheads in compression.

\section{Conclusion}
In this paper we developed the first algorithm capable of performing tractable one-shot entropy coding of large edge-permutation invariant graphs: \emph{Random Edge Coding} (REC).
We provide an example use case with the self-reinforcing Pólya's Urn model \cite{mahmoud2008polya} which performs competitively with state of the art methods despite having $0$ parameters and being void of floating-point arithmetic.

Our method can be seen as an extension of the Bits-Back with ANS method \cite{townsend2019} that is applicable to one-shot compression, i.e., when only a single sample from the data distribution is available.
REC can be trivially extended to hypergraphs and directed edges as mentioned in \Cref{sec:hypergraphs}.

Any model satisfying edge-permutation invariant (\Cref{def:fex}) can be used for optimal compression with REC, including neural network based models.
Learning exchangeable models has been explored in the literature \cite{niepert2014exchangeable, bloem2020probabilistic} but, to the best of our knowledge, using them for compression of graphs and other structured data is an under-explored field.

Pólya's Urn satisfies edge-permutation invariance through the invariance of the PMF to permutations of the vertices, which is a sufficient, but not necessary, condition.
An interesting direction to investigate is if there are similar models that are strictly edge-permutation invariant, that is, the PMF is invariant to permutations of edges and vertices within an edge, but not to permutation of vertices from different edges.
\Cref{appendix:other-fexg} discusses this direction further.

In general, a trade-off exists between the model performance and the complexity required to compute the conditional distributions.
\changed{Pólya's Urn model} lies on an attractive point of this trade-off curve, but there might exist other methods that perform better without increasing complexity significantly.
We think this is a promising line of work that can yield better likelihood models for network data and can provide a principled approach to lossless compression of these data types.

\bibliography{main}
\bibliographystyle{icml2023}

\newpage
\appendix
\onecolumn
\changed{
\section{Proofs}
\subsection{Pólya's Urn Model is Edge-Permutation Invariant (\Cref{def:fex})}\label{appendix:proof-polya}
In this section we prove Pólya's Urn model is vertex-permutation invariant (a sufficient, but not necessary, condition for edge-permutation invariance) and thus can be used for optimal compression with Random Edge Coding.
The proof follows from directly computing the joint $\Pr(v^k)$ from the conditionals $\Pr(v_{i+1} \given v^i)$ and showing that it depends on factors that are invariant to permutations of the vertices in $v^k$.

The conditional is proportional to the degree and is biased by the parameter $\beta$ (which in our case was fixed to $\beta=1$),
\begin{align}
    \Pr(v_{i+1} \given v^i) \propto d_{v^i}(v_{i+1}) + \beta,
\end{align}
with normalizing constant
\begin{align}
    \sum_{v_{i+1} \in [n]} \left(d_{v^i}(v_{i+1}) + \beta\right) = i + n\beta.
\end{align}
The joint is defined as the product of the conditionals,
\begin{align}
    \Pr(v^k) = \prod_{i=0}^{k-1} \frac{d_{v^i}(v_{i+1}) + \beta}{i + n\beta}.
\end{align}
At step $i$, the generative process appends a vertex to the existing sequence resulting in the product of non-decreasing degrees (plus the bias $\beta$) in the numerator.
We can regroup the degree terms and rewrite it as a function of the final degree $d_{v^k}$.
For example, consider the following sequence and its joint distribution
\begin{align}
    v^k &= 12\ 23\ 21 \\
    \Pr(v^k) 
    &= \frac{\overbrace{\beta}^{v_1=1}}{n\beta}\cdot
       \frac{\overbrace{\beta}^{v_2=2}}{1 + n\beta}\cdot
       \frac{\overbrace{1 + \beta}^{v_3=2}}{2 + n\beta}\cdot
       \frac{\overbrace{\beta}^{v_4=3}}{3 + n\beta}\cdot
       \frac{\overbrace{2 + \beta}^{v_5=2}}{4 + n\beta}\cdot
       \frac{\overbrace{1 + \beta}^{v_6=1}}{5 + n\beta} \\
    &= \frac{1}{\prod_{i=0}^{k-1}(i + n\beta)}\cdot
       \underbrace{\beta\cdot(1 + \beta)}_{v_1=v_6=1}\cdot
       \underbrace{\beta\cdot(1 + \beta)\cdot(2 + \beta)}_{v_2=v_3=v_5=2}\cdot
       \underbrace{\beta}_{v_4=3}.
\end{align}

In general, the joint takes on the form
\begin{align}
    \Pr(v^k) 
    &= \frac{1}{\prod_{i=0}^{k-1}(i + n\beta)}\prod_{v \in [n]} (\beta)(1 + \beta)\dots(d_{v^k}(v) - 1 + \beta) \\
    &= \frac{1}{(n\beta)^{\uparrow k}}\prod_{v \in [n]} \beta^{\uparrow d_{v^k}(v)}, \\
\end{align}
where $x^{\uparrow k}=x(x+1)\dots(x+k-1)$ is the ascending factorial.
The expression on the right is clearly invariant to permutations of the elements in $v^k$, which concludes the proof.
}

\changed{
\section{Other Edge-Permutation Invariant Models}
\label{appendix:other-fexg}
\subsection{Extended Pólya's Urn Model}
Note that assigning a unique bias $\beta_v$ to each vertex $v$ will not break the edge-permutation invariance of the Pólya's Urn model, as can be seen from
\begin{align}
    \Pr(v_{i+1} \given v^i) =\frac{d_{v^i}(v_{i+1}) + \beta_{v_{i+1}}}{i + \sum_{v \in [n]}\beta_v},
\end{align}
with corresponding joint distribution
\begin{align}
    \Pr(v^k) 
    &= \frac{1}{(\sum_{v \in [n]}\beta_v)^{\uparrow k}}\prod_{v \in [n]} (\beta_v)^{\uparrow d_{v^k}(v)}. \\
\end{align}
The parameters $\{\beta_v\}_{v \in [n]}$ can be learned via gradient descent methods as the gradient of the joint is easily computable.
However, this model has $n$ parameters, one for each vertex, which would need to be transmitted together with the model depending on how the model generalizes as the network grows.
We did not explore this direction.
}

\changed{
\subsection{Strictly Edge-Permutation Invariant Models via nested Deep Sets \cite{zaheer2017deep}}
In this section we outline a general framework based on \cite{zaheer2017deep, hartford2018deep, meng2019hats} to construct joint distributions that are edge-permutation (EPI) invariant, but are not invariant to permutations of vertices between different edges.
We refer to these models as \emph{strictly} EPI.

Let $\psi\colon [n]^2 \mapsto \mathbb{R}$ and $\Psi\colon \bigcup_{k \in \mathbb{N}} \mathbb{R}^k \mapsto \mathbb{R}^+$ be functions invariant to permutation of their arguments, possibly parameterized by some neural network.
The following joint distribution is clearly EPI, 
\begin{align}
    \Pr(v^{2m}) \propto \Psi(\psi(v_1, v_2), \psi(v_3, v_4), \dots, \psi(v_{2m-1}, v_{2m})).
\end{align}
However, the joint distribution may not be invariant to permutations of vertices between different edges (as intended, making it strictly EPI).
As a concrete example, take
\begin{align}
    \psi(v, w) &= \langle\theta_v, \theta_w\rangle \\ 
    \Psi(\phi_1, \dots, \phi_{2m}) &= \sum_{i \in [2m]} \exp(\phi_i),
\end{align}
where $\theta_v, \theta_w \in \mathbb{R}^\ell$ are embeddings that can be learned and $\langle\cdot,\cdot\rangle$ is the inner-product.

To apply Random Edge Coding, we need to define the conditional distributions
\begin{align}
    \Pr(v_{2i}, v_{2i-1} \given v^{2(i-1)}) = \frac{\Pr(v^{2i})}{\sum_{v_{2i}, v_{2i-1}}\Pr(v^{2i})}.
\end{align}
This model can also be learned via stochastic gradient descent but quickly becomes intractable in the form presented for graphs with millions of edges.
We did not explore this direction.
}

\end{document}